\documentclass[runningheads,envcountsame]{llncs}
\usepackage[T1]{fontenc}
\usepackage{graphicx}
\usepackage{times}
\usepackage{xspace}
\usepackage{soul}
\usepackage{url}
\usepackage{hyperref}
\usepackage{cite}
\usepackage[utf8]{inputenc}
\usepackage{booktabs}
\usepackage{algorithm}
\usepackage{algorithmic}
\usepackage[switch]{lineno}
\usepackage{dsfont}
\usepackage{breakcites}
\usepackage{amsxtra, amsfonts, amssymb, amstext, mathtools}
\usepackage{lineno}
\usepackage{tikz}
\usetikzlibrary{arrows}
\usepackage{xcolor}
\usepackage{tikz}
\usepackage{pgfplots}
\pgfplotsset{compat=1.15}
\usepackage{booktabs}
\usepackage{nicefrac}
\usepackage{xspace}
\usepackage{url}
\usepackage{graphics,color}
\usepackage[algo2e,ruled,vlined,linesnumbered]{algorithm2e}
    \SetKwInOut{Input}{Input}
    \SetKwInOut{Output}{Output}
    \ResetInOut{output}
    \SetKw{Forever}{forever}
\usepackage{wrapfig}
\usepackage{lmodern}
\usepackage{enumerate}

\newtheorem{observation}{Observation}

\newcommand{\oea}{\mbox{${(1 + 1)}$~EA}\xspace}

\newcommand{\opllga}{\mbox{${(1+(\lambda,\lambda))}$~GA}\xspace}
\newcommand{\ollga}{\opllga}

\newcommand{\mupoga}{$(\mu + 1)$~GA\xspace}
\newcommand{\tpoga}{$(2 + 1)$~GA\xspace}

\newcommand{\OM}{\textsc{OM}\xspace}
\newcommand{\onemax}{\textsc{OneMax}\xspace}
\newcommand{\LO}{\textsc{LO}\xspace}
\newcommand{\leadingones}{\textsc{Leading\-Ones}\xspace}

\newcommand{\jump}{\textsc{Jump}\xspace}

\newcommand{\N}{\ensuremath{\mathbb{N}}} 

\newcommand{\eps}{\varepsilon} 

\definecolor{cqcqcq}{rgb}{0.7529411764705882,0.7529411764705882,0.7529411764705882}
\definecolor{aqaqaq}{rgb}{0.6274509803921569,0.6274509803921569,0.6274509803921569}
\definecolor{ffqqqq}{rgb}{1,0,0}
\definecolor{rvwvcq}{rgb}{0.08235294117647059,0.396078431372549,0.7529411764705882}

\SetCommentSty{mycommfont}

\newcommand{\EF}{\mathrm{EF}}

\newcommand{\E}{\mathbb{E}}
\newcommand{\In}{\mathrm{in}}
\newcommand{\Out}{\mathrm{o}}

\newcommand{\Succ}{\mathrm{Succ}}
\newcommand{\ESucc}{\mathrm{ESucc}}
\newcommand{\fit}{\mathrm{fit}}
\newcommand{\val}{\mathrm{val}}

\newcommand*\xbar[1]{%
  \hbox{%
    \vbox{%
      \hrule height 0.5pt 
      \kern0.5ex
      \hbox{%
        \kern-0.2em
        \ensuremath{#1}%
        \kern-0.2em
      }%
    }%
  }%
} 

\begin{document}

\title{How Population Diversity Influences the Efficiency of Crossover}
\titlerunning{cGA on LeadingOnes}
%
\author{Sacha Cerf\inst{1} \and
Johannes Lengler\inst{2}}
\authorrunning{S. Cerf, J. Lengler}
%
\institute{Ecole Polytechnique, Paris, France \and
ETH Z\"urich, Z\"urich, Switzerland}
\maketitle              
%


\begin{abstract}

Our theoretical understanding of crossover is limited by our ability to analyze how population diversity evolves. In this study, we provide one of the first rigorous analyses of population diversity and optimization time in a setting where large diversity and large population sizes are required to speed up progress. We give a formal and general criterion which amount of diversity is necessary and sufficient to speed up the $(\mu+1)$ Genetic Algorithm on \textsc{LeadingOnes}. We show that the naturally evolving diversity falls short of giving a substantial speed-up for any $\mu=o(\sqrt{n})$. On the other hand, we show that even for $\mu=2$, if we simply break ties in favor of diversity then this increases diversity so much that optimization is accelerated by a constant factor.
\end{abstract}

\section{Introduction}

One of the central aspects of genetic algorithms (GAs) is their ability to recombine existing solutions via crossover. This is considered crucial and important in practical applications~\cite{whitley2019next}. In order for crossover to be helpful, it is vital that the population remains diverse, which gives a very specific setting for the exploration/exploitation dualism. Unfortunately, our ability to mathematically analyze population diversity and its impact on runtime has been limited to situations of small populations and/or small diversity, as we will review below. To already give one example, in easy hillclimbing settings like the \onemax problem\footnote{For $x\in \{0,1\}^n$, the \onemax function is defined via $f(x) = \sum_{i=1}^n x_i$.}, a tiny Hamming distance of $2$ between two parents of equal fitness is already beneficial for crossover. In such situations, crossover has been proven to be helpful~\cite{Sudholt17}. 

In this paper, we will treat a situation that was not amenable for analysis with previous techniques, because crossover will only be beneficial if the population diversity is quite large. More precisely, we will study the \leadingones function $\LO(x)$, which returns for $x\in\{0,1\}^n$ the number of one-bits before the first zero-bit in $x$, see Section~\ref{sec:LO} for the formal definition. For a string $x$ with $\LO(x)=k$, in order to improve its fitness it is necessary to flip the $(k+1)$st bit in $x$. Thus, it is rather hard to find such an improvement. \onemax and \leadingones are the most common theoretical benchmarks for hillclimbing in discrete search spaces, where \onemax is supposed to be particularly easy\footnote{In fact, it can be mathematically proven that \onemax is the easiest problem with unique optimum for many algorithms~\cite{doerr2012multiplicative,sudholt2012new,witt2013tight,jorritsma2023comma}.} and \leadingones is designed to be particularly hard~\cite{Rudolph97}. By construction of \leadingones, a crossover between two bit-strings $x$ and $y$, where $\LO(x)=k$, can only be fitter than $x$ if the two parents differ specifically in the $(k+1)$st position. This is a quite strong requirement and this situation will usually only occur if the population is very diverse. In contrast, on \onemax an offspring of $x$ and $y$ can be fitter than $x$ if the parents differ in \emph{any} position where $x$ has a zero-bit, which happens even with minimal population diversity. The main contribution of this paper is that we develop a method to track the population diversity even if it is large, and that we give a general criterion to translate population diversity into runtime\footnote{We measure the runtime as the number of function evaluations until the optimum is evaluated.} for the \mupoga on \leadingones.

\subsection{Our Results}
We analyze the runtime of the elitist $(\mu+1)$ Genetic Algorithm (or \mupoga for short) on \leadingones. We use standard options for mutation and crossover operators: standard bit mutation with rate $\chi/n$ for some constant $\chi>0$ and uniform crossover with uniform parent selection, see Section~\ref{sec:algorithm} for details. It was known before that without crossover, the expected runtime of the \mupoga on \leadingones is $(1+ o(1))\frac{e^{\chi}-1}{2\chi^2}n^2$ for all $\mu =o(n/\log n)$~\cite{BottcherDN10,Witt06}.

Our first result is that this runtime stays the same for the \mupoga for any $\mu = O(\sqrt{n}/\log^2 n)$, up to a $(1+o(1))$ factor.\footnote{For ease of terminology we will ignore $(1+o(1))$ factors in the rest of this exposition.} The core contribution of the proof lies in showing that the population diversity, measured by the average Hamming distance of two randomly selected parents, is bounded by $O(\mu)$. We show in a general setting that this diversity is too small to speed up the runtime by any constant factor.
Our technique builds on a recent result by Jorritsma, Lengler and Sudholt~\cite{jorritsma2023comma}, who analyzed how population diversity of the \mupoga evolves in the absence of selective pressure, i.e., for a flat fitness function. Hence, for moderately large population sizes, the \mupoga lacks population diversity.

Our second result shows that this problem can be overcome easily, as it is rather easy to increase population diversity. If we simply break ties between equally fit individuals in favor of diversity, then even for $\mu=2$ the average Hamming distance increases to $\Omega(n)$. This speeds up optimization by a constant factor.
%
%
%

\subsubsection*{Intuitive Explanation of the Results}

\paragraph{Preparation: runtime without crossover.} Let us first recapitulate where the runtime for $\mu=1$ comes from (without crossover, as this does not make sense for $\mu=1$). When the current search point $x$ has fitness $\LO(x)=k$, then for an improvement it is necessary to flip the $(k+1)$st bit of $x$, which happens with probability $\chi/n$. The expected time until this happens is $n/\chi$. There is a second condition: the bits $1,\ldots,k$ must not be flipped. It can be shown that this second condition leads to an aggregated factor of $(e^{\chi}-1)/\chi$. This is not completely obvious, but is also not hard with the modern tools of \emph{drift analysis} that have been developed in the last decade~\cite{lengler2020drift}. Moreover, the second condition will not interfere with the effect of crossover, so we will ignore this technical aspect in the intuitive explanation below.

The two aforementioned conditions for fitness improvement would lead to a runtime of $(1+ o(1))\frac{e^{\chi}-1}{\chi^2}n^2$ if it was necessary to visit all $n$ fitness levels. However, this is not necessary. When the $(k+1)$st bit is flipped, then it may happen by chance that the $(k+2)$nd bit is already set to one, in which case the algorithm will skip fitness level $k+1$. This happens with probability $1/2$, and in this case the $(k+2)$nd bit is called a \emph{free rider}. There can be more than one free rider at once, and the number of free riders is well-understood: it is geometrically distributed with parameter $1/2$, and the expected number of free riders in a fitness improvement is $1$.\footnote{The geometric distribution is truncated at $n-k$, so the expectation is slightly lower than $1/2$, but this subtlety does not affect the main order term of the runtime.} Hence, in expectation only every second fitness level is visited, which reduces the runtime by a factor of $2$, and leads to the overal runtime of $(1+ o(1))\frac{e^{\chi}-1}{2\chi^2}n^2$.

Without crossover, the above explanation remains essentially unchanged for larger $\mu$, as long as $\mu=o(n/\log n)$. The reason is that once the first individual reaches fitness level $k$, it only takes time $\Theta(\mu \log \mu) = o(n)$ until all individuals are on this fitness level. This time is negligible compared to the time that is needed for the next improvement. Once all individuals have reached fitness $k$, all parents have the same chance to produce an offspring of larger fitness, so the effect of the larger population size is negligible. The discussion up to this point was known from previous work. 

\paragraph{Extra free riders through crossover.} Our main insight lies in the following. With crossover there is an additional chance to make progress. Consider the situation that the whole population is at fitness level $k$, and an offspring $x$ reaches a new fitness level for the first time. Assume for simplicity that there are no free riders in this step, so $\LO(x) = k+1$, although the effect also exists when free riders are present. Then $x$ has a one-bit at position $k+1$ and a zero-bit at position $k+2$. All other individuals have a zero-bit at position $k+1$ because they all have fitness $k$. But it is possible that there is another individual $y$ which has a one-bit at position $k+2$. If $x$ and $y$ perform a crossover, then there is a chance of $1/4$ that it gets the one-bit at position $k+1$ from $x$, and the one-bit at position $k+2$ from $y$, i.e., that it combines the best from the two parents. This effectively gives an \emph{extra free rider}. If this scenario happens for a constant fraction of all levels, this reduces the runtime by a constant factor.

There are two key question for the runtime analysis:
\begin{enumerate}
    \item Conditional on $\LO(x)=k+1$, how likely is it that there is an individual $y$ in the population with a one-bit in position $k+2$?
    \item If there exists such $y$, how likely is it that $y$ transfers its gene to $x$ before it is replaced by individuals of higher fitness?
\end{enumerate}

The answer to the second question is more positive than might seem on first glance, because in each generation the probability that $y$ passes on its gene is only $O(1/\mu)$. However, in order to replace the old population by fitter individuals, the algorithm needs some time: it must select $x$ or its equally fit descendants at least $\mu$ times. (Here we omit the unlikely case that the level is reached a second time by mutation.) Intuitively, this corresponds to $\mu$ chances to select $y$ as the second parent and perform the gene transfer, each with probability $1/\mu$. The real situation is more complex since $y$ could be replaced earlier, but it suffices if the gene continues to exist in the population until half of the population has reached fitness at least $k+1$. In this case, it already has a chance of $\Omega(1)$ to be passed on in form of an extra free rider. We will not need this argument directly for the proofs, but believe that it provides the right intuition: genes that exist are efficiently transferred into extra free riders.

\paragraph{Connection to diversity.} For point 1, recall that $\LO(x)=k+1$ means in particular that $x$ has a zero-bit at position $k+2$. This is key to the situation: we want to obtain a one-bit in a specific position where $x$ can not provide the one-bit by itself. Thus, the probability to understand is: how likely is it that the bit value of $y$ differs from the bit value of $x$ in position $k+2$? This is closely connected to the Hamming distance between $x$ and $y$ and thus, to the diversity. In fact, the \leadingones function has a high level of symmetry, and the bits $k+2,\ldots,n$ do not have any effect on the fitness before the creation of $x$. Hence, if $x$ and $y$ have Hamming distance $d$, then the bits in which they differ are uniformly at random among $k+2,\ldots,n$ (plus the two special position $k$ and $k+1$). Thus, we can compute the probability that $x$ and $y$ differ in position $k+2$ from their Hamming distance, which is directly connected to the population diversity. 

Let us quantify the effect as a function of $\mu$. Once a new fitness interval is reached, the old population is replaced, which represents a genetic bottleneck that reduces diversity. Afterwards, the average Hamming distance starts growing again. If given enough time, it will grow until it reaches $\Theta(\mu)$, at which point it maxes out because diversity may also get lost again whenever individuals are removed from the population. These \emph{equilibrium dynamics} were recently discovered and quantified in~\cite{jorritsma2023comma}. For $\mu = o(\sqrt{n})$, this means that the average Hamming distance stays at $\Theta(\mu)$, and the probability that a fixed individual $y$ differs in position $k+2$ from $x$ is only $O(\mu/n)$. By a union bound over all $\mu$ individuals, the probability that the desired one-bit exists somewhere in the population is at most $O(\mu^2/n) = o(1)$. Since this one-bit typically does not exist, crossover has no chance of providing an extra free rider. We prove this formally, where for technical reasons we make the slightly stronger assumption $\mu = o(\sqrt{n}/\log^2 n)$.

If we modify the \mupoga to break ties in favor of larger population diversity, then the population dynamics changes. We analyze this case for $\mu=2$, where the equilibrium shifts from $\Theta(1)$ to $\Theta(n)$. Moreover, the time required to reach diversity $\Theta(n)$ is only $O(n)$. This is quick enough to give on average a constant number of extra free riders per fitness level, which leads to a constant factor speed-up. 

Although we do not examine the case in this paper, let us briefly speculate on the case $\mu = \Omega(\sqrt{n}) \cap o(n/\log n)$, without modification of the tie-breaking rule. This may look promising since the aforementioned equilibrium dynamics remain true: the average Hamming distance is $\Theta(\mu)$, so it seems conceivable that point 1 from above has a high probability. However, we conjecture that this is not the case, and that the probability is $o(1)$, because the diversity is generated by $o(n)$ positions who differ in many pairs, while $n-o(n)$ positions are identical throughout the population. Nevertheless, we believe that this setting is worth exploring, since a mechanism for increasing diversity in this regime could potentially lead to runtime $o(n^2)$. We leave the exploration of this regime to future work.

\subsection{Related Work}
 

There is a very long history of theoretical work on crossover, and we can only give a brief overview. For a thorough overview of the theoretical study of population diversity we refer to the review by Sudholt~\cite{sudholt2020benefits}. For the more specific question how diversity can provably decrease runtime, a more detailed discussion can be found in~\cite{DoerrEJK23arxiv}.

Our result on \leadingones is by far not the first setting in which crossover is provably beneficial. Historically one of the first rigorous mathematical results were for functions that were specifically tailored to make crossover beneficial, such as the \textsc{RealRoyalRoad} function~\cite{jansen2005real}. A non-tailored example is the \onemax function mentioned above. Sudholt \cite{Sudholt17} proved that crossover accelerates a non-standard version of the \tpoga by a constant factor on \onemax. Subsequent work by Corus and Oliveto~\cite{CorusO18tec} showed that a constant factor speedup is also obtained for the standard \tpoga. However, their analyses rely on the fact that crossover between \emph{any} two different search points is helpful for \onemax. So it sufficed to show that the diversity is not literally zero. Experiments in \cite{CorusO18tec} indicated that larger population sizes than $2$ might be helpful, but so far it could not be mathematically shown that higher population sizes $\mu =\omega(1)$ (or even $\mu > 2$) leads to substantially larger diversity that speeds up optimization on \onemax.

Another important benchmark problem is the \jump function, where the optimum is surrounded by a fitness valley of size $k$. There has been a long and rich line of research for crossover on this function, particularly on the \mupoga and some variations~\cite{dang2017escaping,DoerrEJK23arxiv,JansenW99,KoetzingSTgecco11xxx,lengler2024tight,oliveto2022tight}. It had been understood early that mutation can increase diversity substantially~\cite{KoetzingSTgecco11xxx}, but it remained unclear how crossover influences the population dynamics. Hence, polynomial runtime bounds independent of $k$ (for constant $k$) could only be shown if crossover happens so rarely that it does not influence the  dynamics of population diversity~\cite{JansenW99,KoetzingSTgecco11xxx}, or if the process is amended with diversity-enhancing mechanisms~\cite{dang2017escaping}. Without such mechanisms and for larger crossover probabilities, analyses were for a long time limited to minimal amounts of diversity~\cite{DoerrEJK23arxiv,oliveto2022tight}. Even so recent results as the work by Doerr, Echarghaoui, Jamal and Krejca from 2023~\cite{DoerrEJK23arxiv} could only make use of Hamming distances of at least one, i.e., the proof relied on showing that crossover is frequently performed between two individuals which are not identical to each other. However, very recently Lengler, Opris and Sudholt~\cite{lengler2024tight} could show a tight bound by showing that the typical Hamming distances are $2k$, which is the maximal possible Hamming distance on the plateau of local optima of \jump. Hence, they could show that the diversity is very close to the theoretical maximum. However, they could only show their result for a modified version of the \mupoga in which the parents produce several offspring at the same time, and proceed with the fittest. Nevertheless, the result is a milestone as is was the first result on \jump in which high amounts of diversity could be analytically quantified for a crossover-based algorithm without explicit diversity-enhancing mechanism. Notably, the result in~\cite{lengler2024tight} built on the same techniques from 2023 in~\cite{lengler2023analysing} that we also build upon.


Other theoretical work has shown benefits of problem-specific crossover operators~\cite{DoerrHK12,OlivetoSHY08}, of special ways of applying crossover as in the successful design of the \ollga~\cite{doerr2015black}, and of crossover that is enhanced by diversity-preserving mechanisms~\cite{dang2017escaping,LehreY11,neumann2011effectiveness}. A discussion of those and further results can be found in~\cite{DoerrEJK23arxiv} and~\cite{sudholt2020benefits}.

\section{Preliminaries}\label{sec:prelims}
In this section, we formally introduce the optimization problem and algorithm studied in the paper. Then, we introduce the notations that will be used in our analysis. We also explain the concept of unbiased operators from~\cite{LehreW12} since some of our results hold for arbitrary unbiased mutation and crossover operators.
\subsection{LeadingOnes}\label{sec:LO}
\leadingones, or short \LO, is the function which assigns to each bit-string $x$ the number of consecutive ones from the left of $x$.

\begin{definition}[\leadingones]
Let $n \in \N$. The \leadingones fitness of a bit-string $x \in \{0, 1\}^n$ is 
\begin{equation*}
    \LO(x) = \leadingones(x) = \sum_{i=1}^n \prod_{j=1}^i x_j = \max\{1 \le i \le n \mid \forall 1 \le j \le i: x_j = 1\}.
\end{equation*}
\end{definition}

We will sometimes write $f$ instead of $\leadingones$ for the sake of conciseness.
 
The \leadingones function is a standard benchmark problem. Besides \onemax, it is the most common hillclimbing benchmark in discrete domains. While \onemax is designed as a particularly easy hillclimbing, \leadingones was designed as a harder hillclimbing task~\cite{Rudolph97}: for any search point other than the optima, there exists only one flip that leads to an improvement of the fitness function (flipping the first 0-bit). 

\subsection{The \mupoga}\label{sec:algorithm}
The $(\mu+1)$ Genetic Algorithm, or \mupoga for short, is described in Algorithm~\ref{alg:mu+1} for arbitrary mutation and (binary) crossover operators and for arbitrary tie-breaking rules. Our main result will use \emph{standard bit-mutation} and \emph{uniform crossover}: standard bit-wise mutation with mutation rate $\chi$ flips every bit of a bit-string independently with probability $\chi/n$, and uniform crossover consists in taking each bit from one of the two parents, with equal probability and independently from each other. Pseudo-code for these operators are in Algorithms~\ref{alg:sbm} and~\ref{alg:unicross}. Moreover, we will always break ties in favor of the offspring except for Section~\ref{subsec:tiebreak}, where we explicitly study a variant of the \tpoga which uses a diversity-increasing tie-breaking mechanism.

Some of our results, in particular in Section~\ref{sec:unbiased-results} are true for more general mutation and crossover operators and tie-breakers. 

\begin{algorithm2e}[t]
\caption{The $(\mu+1)$-GA for maximizing a fitness function $f$.}
\label{alg:mu+1}
$t \leftarrow 0$; Generate initial population $P_0 \in (\{0, 1\}^n)^\mu$.\\
\Repeat{\Forever}{
    With probability $p_c$, choose a random parent pair of parents in $P_t$ which do not have the same index and generate the offspring $y$ via crossover. Otherwise, choose a random parent $x \in P_t$ and copy it to get $y$. 
    
    Apply mutation on $y$ to get $y^\prime$.
    
    Choose $z$ uniform at random among the individuals in $P_t$ with minimal fitness.

    \If{ $f(y^\prime) > f(z)$}{
         $P_{t+1} \leftarrow P_t \backslash \{z\} \cup \{y^\prime\}$.
    }
    \If{ $f(y^\prime) = f(z)$ \text{and tie-breaker decides for $y^\prime$}}{
         $P_{t+1} \leftarrow P_t \backslash \{z\} \cup \{y^\prime\}$.
    }
    $t \leftarrow t+1$
}
\end{algorithm2e}


\begin{algorithm2e}[t]
\caption{Standard bit-wise mutation}
\label{alg:sbm}
\Input{$y$, bit-string of size $n$}
\For{$1 \le i \le n$}{
With probability $p$, set $y_i = 1-y_i$
}
\end{algorithm2e}

\begin{algorithm2e}[t]
\caption{Uniform crossover}
\label{alg:unicross}
\Input{$x_1, x_2$, bit-strings of size $n$}
Initialize a string $y$ of size $n$

\For{$1 \le i \le n$}{
With probability $\frac{1}{2}$, set $y_i = x_{1_i}$

Otherwise, set $y_i = x_{2_i}$
}

\Return{y}

\end{algorithm2e}

\subsection{Notation}\label{sec:notation}

\textbf{General notation.} We consider the \leadingones function on the search space $\{0, 1\}^n$ for $n\to\infty$, and all Landau notation like $O(.), \Omega(.), \ldots$ is with respect to this limit. We denote search points by $x = (x_1, ..., x_n)$. For any two search points $x, y$, the \emph{Hamming distance} $H(x, y)$ of $x$ and $y$ is the number of positions $1 \le i \le n$ such that $x_i \neq y_i$. For $x \in \{0, 1\}^n$, we define $0_x = \{1 \le i \le n \mid x_i = 0\}$, and $1_x = \{1 \le i \le n \mid x_i = 1\}$. For $S \subseteq \{1, \dots, n\}$, we define $x_S = (x_i)_{i \in S}$. For the special case where $S$ is an integer interval with lower and upper bounds $m$ and $M$, we may write $x_{[m:M]}$, or even $x_{[m:]}$ if $M=n$. Finally, for two random variables $X$ and $Y$, we write $X \preceq Y$ for "$X$ is stochastically dominated by $Y$". We use the same notation if $Y$ is a probability distribution. We write $\mathcal{G}(p)$ for the geometric distribution with mean $1/p$.\medskip

\textbf{\mupoga process.} We denote by $\chi$ the expected number of bits flipped by any mutation operator, and assume $\chi = \Theta(1)$ throughout the paper. For $t \in \N^*$, $P_t = \{x^t_1, \dots, x^t_\mu\} \in (\{0, 1\}^n)^\mu$ is the population after iteration $t$. $P_0$ is the population after initialization. The best fitness of an individual in $P_t$ is denoted $F_t = \max \{\LO(x),$ $x \in P_t\}$. Any individual in $P_t$ that has fitness $F_t$ is called \emph{fit}. An offspring will be called "valid" if it is included in the population at the next iteration. The event "a valid offspring is produced at timestep $t$" will be denoted by $V_t$. For any $x, y \in P_t$, we say that "x produces an offspring" (or "x and y produce an offspring") if the offspring $y^\prime$ in Algorithm \ref{alg:mu+1} is derived by mutation from a copy of $x$ (in case of no crossover) or from the crossover between $x$ and $y$ (in case of crossover). We call $C_t$ the event "a crossover was performed at iteration $t$", and $M^i_t$ the event "bit $i$ is mutated at time $t$". 

For $0 \le i \le n$, we will use the expression "level $i$ is \emph{reached}" to signify "$\exists t \in \N, F_t = i$". This event will also sometimes be called $R_i$. We denote by $T^\In_i$ the time at which we reach fitness level $i$, by $T^\Out_i$ the time spent under or at fitness level $i$, and by $\Succ_i$ the fitness level reached after crossing or leaving fitness level $i$, that is $F_{T^\Out_i + 1}$. We say that a population $P$ is \emph{consolidated} if all individuals in $P$ have the same fitness $i$, and in this case we write $f(P)=i$ by slight abuse of notation. We call the \emph{consolidation time} for fitness level $i$ the time from reaching this level until a consolidated population on this fitness level occurs, $T^c_i = \min \{t \ge T^\In_i \mid \forall x \in P_t, f(x) = i\}$. Note that there may be no consolidated population on fitness level $i$. In that case, we will say that $T^c_i = T^o_i$ by convention. As in ~\cite{DrosteJW02}, "fitness level $i$ is \emph{essential}" means that it is reached, and left by mutation, i.e $R_i$ is realized and $C_{T^\Out_i} = 0$, or $R_i$ is realized and $C_{T^\Out_i} = 1$, but the intermediate offspring $y$ generated by crossover at $T^\Out_i$ has fitness lower than $i+1$ (the improvement is brought by the mutated final offspring $y^\prime$). We call this event $E_i$. If an essential fitness level $i$ is left before \emph{consolidation} ($T^c_i > T^\Out_i$), it is called "strange", which is denoted $S_i$. An essential fitness level $i$ which is not strange is called "normal", which is denoted by $N_i$. We also denote $\ESucc_i$ the smallest essential fitness level after $i$, i.e. $\ESucc_i := \min \{j > i \mid E_j\}$. We set $\ESucc_i := n+1$ if there is no $j>i$ with $E_j$.\medskip

\textbf{Diversity measure.} Following ~\cite{lengler2023analysing}, we define for a population $P$ the sum of pairwise Hamming distances in the population $ S(P) = \sum_{x \in P} \sum_{y \in P} H(x, y)$, and for $x \in \{0, 1\}^n$, the sum $S_P(x) =  \sum_{y \in P} H(x, y)$ of Hamming distances between $x$ and all individuals in $P$. Note that the average Hamming distance between two individuals (without repetition) is $\frac{S_P(x)}{\mu(\mu-1)}$. When it is clear, we will omit the index of the population $P$ we are summing over: $S_P(x) = S(x)$. 
Finally, we call "diversity of the population at time $t$", the quantity $d_t = \frac{S(P^\prime_t)}{\mu(\mu - 1)(n-F_t - 1)}$, where $P^\prime_t = \{x_{[F_t + 2:]}, x \in P_t \}$. This is the average pairwise Hamming distance of the non-optimized parts of the bit-strings in the population (not counting the bit just after the current fitness level), normalized by the size of the non-optimized part of a fit bit-string.\medskip

\textbf{Runtime.} For $\mu \in \N^*$, the \emph{runtime} $T^\mu = T^\mu_n$ of the \mupoga on \leadingones is the number of function evaluations before the optimum is found. 

\subsection{Unbiased offspring generation mechanisms}
\label{subsec:unbiased}
Our analysis builds on the fact that, for a given individual in the population, the bits between $F_t + 2$ and $n$ are uniformly distributed in the space of bit-strings of size $n - F_t - 1$, and that the bits in which two individuals differ are evenly distributed in this range. As we will show below, this is generally true if mutation, crossover, and tie-breaker are \emph{unbiased} operators. The notion of unbiased operators has been introduced in~\cite{LehreW12} as operators which are invariant under automorphisms of the hypercube. 
%
We recall the definition of unbiased operators from~\cite{lengler2023analysing}. Since the group of automorphisms is generated by permuations and by applying XOR with fixed bit-strings, it suffices to require the following two conditions.
\begin{definition}
    Let $\psi : (\{0, 1\}^n)^k \rightarrow \{0, 1\}^n$ be \emph{k-ary operator}. $\psi$ is \emph{unbiased} if the following two conditions hold for all $y, x_1, \dots, x_k \in \{0, 1\}^n$
    \begin{enumerate}[(i).]
        \item For every permutation of $n$ positions $\sigma$, $$\Pr(\psi(x_1, \dots, x_k) = y) =\Pr(\psi(\sigma(x_1), \dots, \sigma(x_k)) = \sigma(y)).$$
        \item For every bit-string $z$ of length $n$,  $$\Pr(\psi(x_1, \dots, x_k) = y) =\Pr(\psi(x_1 \oplus z, \dots, x_k \oplus z) = y \oplus z). $$
    \end{enumerate}
\end{definition}

The unbiased framework has been very successful especially in the context of black-box complexity~\cite{doerr2020complexity}. 
Most of the standard mutation operators, such as standard bit-wise mutation or the heavy-tailed mutation used in \emph{fast GAs}~\cite{DoerrLMN17} are unbiased. Many crossover operators are unbiased, like uniform crossover, but some, like the single-point crossover, are not, see~\cite{FKRSSTW22} for details. Our default in this paper are standard bit-wise mutation and uniform crossover, both of which are unbiased. Note that a tie-breaker can be considered as a $(\mu+2)$-ary operator, taking as an input the full population of size $\mu$ and two additional search points between which we want to break ties, one of which it needs to return. 

    
For the \mupoga, we call the combination of crossover (if applied), mutation and tie-breaking as \emph{offspring generation mechanism}, and we call such a mechanism unbiased if all three operators that constitute it are unbiased. We call the algorithm \mupoga unbiased if its offspring generation mechanism is unbiased and if it is initialized with a $0$-ary unbiased operator. We may now prove some useful results that are true for an unbiased \mupoga on the \leadingones problem.
%
Define, for every automorphism of the hypercube $\pi$, and every population $P$, $\pi(P) = \{\pi(x), x \in P \}$, and $F(P) = \max\{\LO(y), y \in P\}$. The following result states that the \mupoga with population $P_t$ is invariant under automorphisms that keep the first $F(P_t)+1$ bits fixed.  
\begin{lemma}
    \label{lem:conditional_popgen}
    Consider a run of the \mupoga using an unbiased offspring generation mechanism. Then for all $t \ge 0$, all $Q, P \in \{0, 1\}^\mu$ and all automorphisms $\pi$ of the hypercube such that for all bit-string $x$, and all $j \le F(P) + 1$, $\pi(x)_j = x_j$,
    $$\Pr(P_{t+1} = Q \mid P_t = P) = \Pr(P_{t+1} = \pi(Q) \mid P_t = \pi(P)).$$
\end{lemma}
\begin{proof}
    Let $\pi$ be as in the lemma. In the notation of Algorithm~\ref{alg:mu+1}, the only part of the algorithm that is not unbiased is the selection of $z$, which is chosen uniformly random among the individuals in $P_t$ with minimal fitness. Denote the set of individuals with minimal fitness in $P$ by $m(P)$. Then we claim
    \begin{equation}
    \label{equiv:least_fit_conservation}
        \pi(m(P)) = m(\pi(P)).
    \end{equation}
    To see this, note that all individuals $x\in P$ have fitness at most $\LO(x)\le F(P)$. Hence, the fitness of $x$ is determined by the bits $x_{[1:F(P)+1]}$. Since these bits are unchanged under $\pi$, we obtain $\LO(\pi(x)) = \LO(x)$ for all $x\in P$. In particular, the set of individuals of minimal fitness in $P$ is unchanged by $\pi$, which implies~\eqref{equiv:least_fit_conservation}. 

    Since $z$ is selected uniformly at random from $m(P_t)$, Equation~\eqref{equiv:least_fit_conservation} implies for all $z_0 \in P$,
    \begin{equation*}
        \Pr(z = z_0 \mid P_t = P) = \Pr(z = \pi(z_0) \mid P_t = \pi(P)).
    \end{equation*}
    
    Since all other steps of Algorithm~\ref{alg:mu+1} are unbiased, they are invariant under arbitrary automorphisms, and in particular they are invariant under $\pi$. This proves Lemma~\ref{lem:conditional_popgen}.\qed

\end{proof}

Since $F(P_t)$ can only increase for the \mupoga, we obtain the following symmetry as an immediate consequence.
\begin{lemma}
\label{lem:uniform_nonopt_pop}
   Consider an unbiased \mupoga. For all $t \ge 0$, all $P \in \{0, 1\}^\mu$ and all automorphism $\pi$ of the hypercube such that for all bit-string $x$, and all $j \le F(P) + 1$, $\pi(x)_j = x_j$, 
    $$\Pr(P_t = P) = \Pr(P_t = \pi(P)).$$
\end{lemma} 
\begin{proof}
    We proceed by recursion over $t$. For $t = 0$, we know that $P_0$ is obtained by a zero-ary unbiased operator, so the proposition is true.
   
    Suppose that the proposition is true for some $t \ge 0$. 
    Then,
    \begin{align*}
& \resizebox{\linewidth}{!}{%
    $\begin{aligned}
        \Pr(P_{t+1} = \pi(P)) &= \sum_{Q \in \{0, 1\}^\mu}\Pr(P_t = Q)\Pr(P_{t+1} = \pi(P) \mid P_t = Q)\\
        &=\sum_{Q \in \{0, 1\}^\mu}\Pr(P_t = Q)\Pr(P_{t+1} = P \mid P_t = \pi^{-1}(Q)) &\text{by Lemma \ref{lem:conditional_popgen},} \\
        &=\sum_{Q \in \{0, 1\}^\mu}\Pr(P_t = \pi^{-1}(Q))\Pr(P_{t+1} = P \mid P_t = \pi^{-1}(Q)) &\text{by recursion hypothesis,}\\
        &=\sum_{R \in \{0, 1\}^\mu} \Pr(P_t = R)\Pr(P_{t+1} = P \mid P_t = R) &\text{setting $R = \pi^{-1}(Q)$,}\\
        &= \Pr(P_{t+1} = P).
    \end{aligned}$}
\tag*{\qed}\end{align*}
\end{proof}

Lemma \ref{lem:uniform_nonopt_pop} has some interesting corollaries, the first one being on the distribution of the non-optimized part of a single bit-string in the population.

\begin{corollary}
    For all $1 \le j \le \mu$, $x^t_{j_{[F_t + 2:]}}$ is uniformly distributed on $\{0, 1\}^{n - F_t - 1}$.
\end{corollary}
\begin{proof}
    Let $1 \le j \le \mu$. Let $s, s^\prime \in \{0, 1\}^{n - F_t - 1}$. Denote by $b$ the bit-string of length $n$ such that $b_i = 0$ for all $i \le F_t + 1$, and $b_{[F_t + 1:]} = s \oplus s^\prime$. Then the automorphism of the hypercube $\pi_b$ defined as $\pi_b(x) = x \oplus b$ for all $x \in \{0, 1\}^n$ satisfies the hypothesis of Lemma \ref{lem:uniform_nonopt_pop}. Moreover, for $y \in \{0, 1\}^{n - F_t - 1}$, if we denote as $\Pi(y)$ the set of populations, the $j$-th element of which has bits from $F_t + 2$ to $n$ equal to $x$, then $\pi_b$ induces a bijection between $\Pi(s)$ and $\Pi(s^\prime)$. It follows from Lemma \ref{lem:uniform_nonopt_pop} that:
    $$\Pr(P_t \in \Pi(s)) = \Pr(P_t \in \Pi(s^\prime)),$$
    which can be rewritten as:
    \begin{align*}
    \Pr(x^t_{j_{[F_t + 2:]}} = s) = \Pr(x^t_{j_{[F_t + 2:]}} = s^\prime).\tag*{\qed}
    \end{align*}
\end{proof}
We also get the following very useful result that helps us bound the size of the fitness jumps.
\begin{corollary}
    \label{cor:small_fitness_improvements}
    For all $t \ge 0$, for all $2 \le j$, $\Pr(F_{t+1} - F_t \ge j \mid F_{t+1} - F_t \ge 1) = 2^{-(j-1)}$.
\end{corollary}
\begin{proof}
    Using a similar argument as in the last proof, and the unbiasedness of $\mathcal{C}$, we simply argue that for all $t \ge 0$, the offspring $y_t$ generated at time $t$ is such that $y_{t_{[F_t + 2:]}}$ is uniformly distributed over the set of bit-strings of size $n - F_t - 1$, even conditioned on the fact that it improves fitness. Hence, the probability that there it has a streak of $j-1$ $1$-bits starting from position $F_t + 2$ is $2^{-(j - 1)}$. \qed
\end{proof}

Finally, we prove this result on the distribution of the bits that differ between two individuals in the population at a given iteration. Recall that $d_t$ is the average density of non-equal bits in the population when we restrict to the non-optimized part of the bit-string $[F_t+1:n]$.
\begin{corollary}\label{cor:unbiased}
    Let $t \ge 0$, and $1 \le i, i' \le \mu$ with $i \neq i'$. Let $d := x^t_{i_{[F_t + 2:]}} \oplus x^t_{i'_{[F_t + 2:]}}$ be the string that has a $1$ where $x^t_i$ and $x^t_{i'}$ are different and a $0$ elsewhere, restricted to the non-optimized part. Then, for any pair of bit-strings $s, s^\prime \in \{0, 1\}^{n - F_t - 1}$ of equal Hamming weight:
    $$
    \Pr(d = s) = \Pr(d = s^\prime).
    $$
In particular, for any $j\in[F_t+2:n]$ and for all $H\ge 0$,
    \begin{align*}
        \Pr\left((x_i^t)_j \neq (x_{i'}^t)_j \mid H(x^t_{i_{[F_t + 2:]}}, x^t_{{i'}_{[F_t + 2:]}})= H\right) & = \frac{H}{n - F_t - 1}.
    \end{align*}
Moreover, if the tie-breaker is symmetric with respect to permutations of the population and the initial population is uniformly random, then
    \begin{align*}
        \Pr((x_i^t)_j \neq (x_{i'}^t)_j \mid d_t) & = d_t.
    \end{align*}
\end{corollary}

\subsubsection{Preliminary results on the consolidation process}\label{sec:unbiased-results} 

The following lemmas treat the time of raising the whole population to fitness at least $i$, after this fitness level has been found. This time is well-known to have expectation $O(\mu\log \mu)$, e.g.~\cite{Witt06}. Here we provide a tail bound. Note that the crossover probability $p_c$ and the parameter $p_{\text{clone}}$ that appears in the following lemma are not included into the index of $C_\beta$ because those are part of the algorithm, which we consider as fixed.

\begin{lemma}\label{lem:consolidation_tailbound}
     Consider the \mupoga with a mutation operator that has probability $p_{\text{clone}}$ of duplicating the parent. Let $1 \le i \le n$ be any fitness level of \leadingones. For any $t \ge 0$, denote by $X_t$ the number of individuals of fitness larger or equal to $i$ in $P_{T^{in}_i + t}$. Then, for any constant $\beta > 0$, $p_{\text{clone}} >0$ and $p_c < 1$, there exists $C_{\beta} > 0$ such that for $n$ big enough the following holds for all $C > C_{\beta}$.
     \begin{align}\label{eq:consolidation_tailbound}
     \Pr(X_{C\mu \log \mu} < \mu) \le \mu^{-\beta}.
     \end{align}

     More precisely, the statement holds for $C \ge (1+\beta)/((1-p_c)p_{\text{clone}})$. For standard bit mutation with constant mutation rate $\chi$, Equation~\eqref{eq:consolidation_tailbound} holds for all $C > (1+\beta)e^{\chi}/(1-p_c)$ if $n$ is sufficiently large.
\end{lemma}

\begin{proof}
    For $1 \le j \le \mu - 1$, let $T_j$ be the random variable denoting the time to go from $j$ to $s > j$ individuals with fitness larger or equal to $i$ (let us call such a fitness "large"). If an individual that has large fitness is chosen for reproduction and makes a copy of itself, the number of individuals with large fitness increases. This happens with probability at least $j(1-p_c)p_{\text{clone}}/\mu$.
    Now, note that for any $C > 0$, 
    $$\Pr(X_{C\mu \log \mu} < \mu) \le \Pr(\sum_{j=1}^{\mu - 1} T_j > C\mu \log \mu).$$
    Finally, we apply Theorem 1.10.35 from \cite{DoerrN20}, which yields~\eqref{eq:consolidation_tailbound} for $C \ge C_\beta := (1+\beta)/((1-p_c)p_{\text{clone}})$. For standard bit mutation with mutation rate $\chi$, the statement follows from $p_{\text{clone}} = (1-\chi/n)^n = (1-o(1))e^{-\chi}$, and hence $(1+\beta)/((1-p_c)p_{\text{clone}}) < C$ for sufficiently large $n$.\qed
\end{proof}

The above lemma immediately translates into a tail bound for the consolidation time $T_i^c$. Recall that this is the first time when either all individuals have fitness at least $i$, or at least one individual has fitter strictly larger than $i$.
\begin{lemma}
\label{lem:stayingisconsolidating}
    In the situation of Lemma~\ref{lem:consolidation_tailbound}, for any constant $\beta > 0$, $p_{\text{clone}} >0$, $p_c < 1$ and $C > C_{\beta} $ and any fitness level $1\le i \le n$, if $n$ is sufficiently large,
    $$
    \Pr(T^c_i - T^{\In}_i > C\mu \log \mu) \le  \mu^{-\beta}.
    $$
\end{lemma}

\begin{proof}
    Let $1 \le i \le n$ be any fitness level. For any $t \ge 0$, denote by $X_t$ the number of individuals of fitness larger or equal to $i$ in $P_{T^{in}_i + t}$. For $0 \le t \le T^c_i - T^{in}_i$, there are no individuals with fitness larger than $i$ in the population, so the population is consolidated on $i$ at time $C\mu \log \mu$ if and only if $X_{C\mu \log \mu} = \mu$. Hence,
    \begin{align*}\Pr(T^c_i - T^{\In}_i > C\mu \log \mu) & = \Pr(X_{C\mu \log \mu} < \mu \cap T^c_i - T^{\In}_i > C\mu \log \mu)  \\
    & \le \Pr(X_{C\mu \log \mu} < \mu).
    \end{align*}
    We conclude by Lemma~\ref{lem:consolidation_tailbound}.\qed
\end{proof}

As mentioned above, the same bound also holds in expectation. This can either be derived from Lemma~\ref{lem:stayingisconsolidating} with a bit of computation or taken from~\cite{Witt06}.
\begin{corollary}
    \label{cor:consolidationexpectation}
    For any fitness level $1 \le i \le n$, $\E[T^c_i - T^{\In}_i] = O(\mu \log \mu)$.
\end{corollary}

\section{Analysis of the \mupoga on \leadingones for different population regimes}
\label{sec:analysis}
In this section, we will first show that the number of extra free-riders determines the expected runtime of the \mupoga. Throughout the section, we will assume that the \mupoga uses standard bit mutation where the mutation rate $\chi$ and the crossover probability $p_c < 1$ are constants. However, at first the algorithm may use any respectful\footnote{In a respectful crossover, if both parents have the same bit at some position $i$, the offspring also has the same bit at position $i$.} unbiased crossover operator and and unbiased tie-breaking rule. 

In Section ~\ref{subsec:vanillanospeedup}, we show that if $\mu = O(\sqrt{n}/\log^2 n)$ then the leading constant of the runtime remains unchanged. However, as we show in ~\ref{subsec:tiebreak}, even for $\mu = 2$, a simple diversity-preserving mechanism suffices to obtain a constant speedup factor.
Our strategy is based on the notion of \emph{extra free-riders} due to crossover, as introduced in Section~\ref{sec:notation}. This term is derived from the term "free-rider" originating in~\cite{DrosteJW02}, defined as follows.
\begin{definition}[free-riders]
    Let $x$ be a bit-string of size $n$, and suppose that
    $\LO(x) = i$  for some $1 \le i \le n-2$. The free-riders of $x$ are the leading ones of the sub-string $(x_{i+2}, \dots, x_n)$.
\end{definition}
In other words, a free-rider of a bit-string $x$ is a bit that automatically becomes a leading one as soon as the first zero bit of $x$ is flipped. In an evolutionary optimization context, these free-riders allow us to skip some fitness levels. In the context of the \mupoga, we prefer to associate free riders with a fitness level $i$ as follows. Recall that $E_i$ is the event that the $i$-th fitness level is essential.

\begin{definition}[Free-riders associated to a fitness level]
    For $0 \le i \le n-1$, we denote as $F_i$ and call "free-riders associated to level $i$" the following random variable.
    \begin{equation*}
     F_i =
    \begin{cases}
      \Succ_i - i - 1 & \text{if $E_i$}\\
      0 & \text{otherwise}
    \end{cases}       
    \end{equation*}
\end{definition}

Extending on this idea, we define "extra free-riders" as additional leading ones that are obtained with crossover. 

\begin{definition}[extra free-riders]
    Let $x$ and $x^\prime$ be two bit-strings of size $n$. Denote $\leadingones(x) = i$, $\leadingones(x') = j$, and suppose $j > i$. Consider the result $y$ of a crossover between $x$ and $x^\prime$, with $\leadingones(y) = k$. The extra free-riders brought by this crossover are the bits $y_{[j+1:k]}$.
   
\end{definition}

Note that if the crossover operator is \emph{respectful}, which is the case for most common crossover operators (see~\cite{lengler2023analysing} for a classification), then the 1-bits at positions $j+1, \dots, k$ come from either $x$ or $x'$. Thus, the reason why we suggest the name "extra free-rider" is because, just like normal free-riders, these 1-bits already \emph{accidentally} exist among the population (that is, they are not here because of an optimization choice of the algorithm, but because of genetic drift), and allow us to overcome some fitness plateaus in negligible time. 

The typical scenario for the acquisition of extra free-riders is that, after the whole population is brought to a common fitness plateau, diversity accumulates on the non-optimized trailing part of the bit-strings. Then, when an individual $x$ reaches a higher fitness level, some individuals in the lower levels may happen to have a $1$-bit at the position corresponding to the next fitness level. If a crossover between one of these individuals and an individual of fitness $f(x)$ is performed, then with a good probability we get extra free-riders. These extra free-riders can be obtained from multiple successive crossovers, until the next fitness level is left via mutation. We introduce a useful definition that stems from this observation. Recall that $\ESucc_i$ is the smallest essential fitness level after level $i$.
\begin{definition}[Extra free-riders associated to a fitness level]
    For $0 \le i \le n-1$, we denote as $\EF_i$ and call "extra free-riders associated to level $i$" the following random variable.
    \begin{equation*}
     \EF_i =
    \begin{cases}
      \ESucc_i - \Succ_i & \text{if $E_i$}\\
      0 & \text{otherwise}
    \end{cases}       
    \end{equation*}
\end{definition}

The two following lemmas draw a link between the expected value of $\EF_i$ for any reached fitness level $i$ in a given implementation of the \mupoga, and the expected runtime of the algorithm. Recall that we consider the \mupoga with standard bit-wise mutation with mutation rate $\chi = \Theta(1)$, $p_c<1$, any respectful crossover operator, and any unbiased tie-breaker, but the proof may be adapted to most known mutation mechanisms.

\begin{lemma}
    \label{lem:freeridersspeedup}
    Consider the \mupoga with standard bit mutation and any respectful unbiased mutation operator. 
    Suppose that there exists a sequence of functions $(m_n)_{n \in \N}$, uniformly convergent to a function $m$, and $k(n) = \omega(1)$, such that $\eps_{m_n} = \max\{|m_n(x) - m_n(y)| \mid |x - y| \le \frac{n}{k}\} = o(1)$, and, for all fitness level $0 \le i \le n-1$, $\Pr(\EF_i \ge 1 \mid N_i) \ge m_n(\frac{i}{n})$. Then:
        $$
        \E[T^\mu] \le \frac{n^2}{\chi}\int_{0}^1\frac{e^{\chi x}}{2 + m(x)}dx + o(n^2).
        $$
\end{lemma}

\begin{lemma}
    \label{lem:freeridersspeeduplb}
   Consider the \mupoga with standard bit mutation and any respectful unbiased mutation operator. Suppose that there exists a sequence of functions $(M_n)_{n \in \N}$ defined on $[0, 1]$, uniformly convergent to a function $M$, and $k(n) = \omega(1)$, such that $\eps_{M_n} = \max\{|M_n(x) - M_n(y)| \mid |x - y| \le \frac{n}{k}\} = o(1)$,  and, for all fitness level $0 \le i \le n-1$,  $\E[\EF_i \mid N_i] \le M_n(\frac{i}{n})$. Suppose also that the event $A :$ "for all $1 \le i \le n, \ESucc_i - i = o(\frac{n}{\max(k, \mu \log \mu)})$" holds with high probability. Then:
        $$
        \E[T^\mu] \ge \frac{n^2 + o(n^2)}{\chi}\int_{0}^1\frac{e^{\chi x}}{2 + M(x)}dx.
        $$
\end{lemma}

Everything remains true if the bound for $\E[EF_i \mid N_i]$ holds for $o(n) \le i \le n - o(n)$, but for the sake of simplicity, we will suppose that the bounds hold for any $i$.
\begin{observation}
    Note that for the case where $m_n$ is constant, the inequality from Lemma \ref{lem:freeridersspeedup} becomes:
    $$
    \E[T^\mu] \le \frac{1}{2 + m_n}\frac{e^\chi - 1}{\chi^2}n^2 + o(n^2),
    $$
    and similarly, for the case where $M_n$ is constant, the inequality from Lemma \ref{lem:freeridersspeeduplb} becomes:
    $$
    \E[T^\mu] \ge \frac{1}{2 + M_n}\frac{e^\chi - 1}{\chi^2}n^2 + o(n^2).
    $$
\end{observation}
\begin{observation}
    \label{obs:oea_case}
    When $p_c = 0$, one can simply set $m_n = M_n = 0$, and obtain: 
    $$
    \E[T^1] = \frac{e^\chi - 1}{2\chi^2}n^2 + o(n^2),
    $$
    which is, up to a negligible term, the runtime proved by ~\cite{BottcherDN10} for the \oea.
\end{observation}

The reader may notice that these two lemmas are stated in a much stronger form that what we will use in this paper, that is, the case where $M_n$ or $m_n$ are constant. This is because we believe they can be a useful tool for whoever would like to study other variants of the \mupoga, where these two quantities might depend on the current fitness level.

To prove these lemmas, we will need two preparatory results. The first one will help us neglect strange fitness levels.

\begin{lemma}
    \label{lem:no_strange}
    For any fitness level $0 \le i \le n-1$,  $\Pr(S_i) = O\left(\frac{\mu \log \mu}{n}\right)$.
\end{lemma}
\begin{proof}

    Let $0\le i \le n-1$ be any fitness level. For any integer $k \ge 0$, denote by $A_k$ the event "between time steps $T_i^{in}$ and $T_i^{in} + k\mu \log \mu$, the population consolidates on $i$", and for $k \ge 1$ denote by $M_k$ the event "between time steps $T_i^{in} + (k-1)\mu \log \mu$ and $T_i^{in} + k\mu \log \mu$, level $i$ is left by mutation" . For $S_i$ to happen, it must be that for some $k \ge 0$, $A_j$ is not realized and $M_{k+1}$ is. Hence, by a union bound,
    \begin{align*}
        \Pr(S_i) &\le \sum_{k=0}^{\infty}\Pr(\bar{A_k} \cap M_{k+1}).
    \end{align*}
    Now, note that during a time frame of $\mu \log \mu$, by a union bound, the probability to flip the $i+1$-st bit is at most $\frac{\chi\mu \log \mu}{n}$, independently of previous steps. Hence, for $k\ge 0$, $\Pr(M_{k+1} \mid \bar{A_k}) \le \frac{\chi \mu \log \mu}{n}$.
    Moreover, by Lemma \ref{lem:stayingisconsolidating} there is a constant $C>0$ such that $\Pr(\bar{A_k}) \le \mu^{-C(k - 1)}$ for all $k \ge 0$.
    We deduce that $\sum_{k=0}^{\infty}\Pr(\bar{A_k} \cap M_{k+1}) \le \sum_{k=0}^\infty \frac{\chi \mu \log \mu}{n}\cdot \mu^{-C(k - 1)}= O(\frac{\mu \log \mu}{n})$, which concludes the proof.  \qed  
\end{proof}

The second preliminary result essentially tells that we need to focus our efforts on studying the time spent on normal fitness levels.

\begin{lemma}\label{lem:decomposition_normal_rest}
    For $\mu = o(n/\log n)$ the following equality holds:
    $$
    \E[T^\mu] = \sum_{i=0}^{n-1} \E[\mathds{1}_{N_i}(T^\Out_i - T^c_i)] + o(n^2).
    $$
\end{lemma}

\begin{proof}
    The only fitness levels $i$ for which $T^c_i < T^\Out_i$ are normal fitness levels. Hence, the runtime of the \mupoga can be decomposed as:
    $$
    T^\mu = \sum_{i=0}^{n-1} \mathds{1}_{N_i}(T^\Out_i - T^c_i) + \sum_{i=0}^{n-1} (T^c_i - T^\In_i).
    $$
    By Corollary \ref{cor:consolidationexpectation}, the second term is at most $O(n\mu \log \mu) = o(n^2)$ in expectation, which concludes the proof.\qed
\end{proof}

With this preparation, we can now prove Lemma~\ref{lem:freeridersspeedup} and \ref{lem:freeridersspeeduplb}. 

\begin{proof}[of Lemma~\ref{lem:freeridersspeedup}]
    By Lemma~\ref{lem:decomposition_normal_rest} we may focus on normal levels. We observe that on any normal fitness level $i$, we have $T^\Out_i - T^c_i \sim \mathcal{G}((1-\frac{\chi}{n})^i\frac{\chi}{n})$. Indeed, from the moment when consolidation happens, every individual in the population has fitness level $i$. Consequently, at each iteration, regardless of which parent(s) we select, and of whether we use crossover or not (since the crossover operator is respectful, it conserves common bits among the parents, in which case the first $i$ 1-bits are always copied to the offspring), the probability of producing an offspring with fitness level higher than $i$ is determined only by the mutation phase, where we have to keep the first $i$ bits untouched (with probability $(1-\frac{\chi}{n})^i$) and mutate the $(i+1)$-st bit (with probability $\frac{\chi}{n}$). 
    
    We also need to know how many normal levels there are in a local window of fitness levels. This number will depend on the expected size of the jumps between essential fitness levels, which in turn depends on the expected number of extra free-riders associated to normal fitness levels. The more extra free-riders we get, the sparser normal fitness levels are, and the faster the optimization is. To make this precise, we will use the following argument: we partition the fitness level in local fitness windows of size $\frac{n}{k}$ (where $k$ is defined in the statement of the lemma), and then, we show that the expected cumulative size of the jumps starting from normal fitness levels inside these windows is roughly $\frac{n}{k}$. Combining this observation with the upper or lower bound on the expected number of extra free-riders associated to normal fitness levels, which is roughly uniform over the local fitness windows, we can count the expected number of normal fitness levels inside the window, which allows us to bound the first term by splitting it into smaller sums over these local windows. 
    Following this idea, let us define, for $0 \le j \le k$, $i_j = \frac{jn}{k}$. Then, for $i_j \le i < i_{j+1}$, define the truncated number of extra free-riders associated to fitness level $i$ as 
    
    $$\widetilde{\EF}_i := \min(i_{j+1}, \ESucc_i) - \Succ_i.$$
    
    We will first argue that for all $0 \le j \le k-1$:
    \begin{equation}
        \label{ineq:local_essentials}
        \sum_{i=i_j}^{i_{j+1}-1} \mathds{1}_{E_i}(1 + F_i + \widetilde{\EF}_i) \le \frac{n}{k}.
    \end{equation}

    Indeed, the sum in the middle is simply $i_{j+1} - L_{j}$, where $L_j$ is the first essential level in $[i_j,i_{j+1}]$, if such a level exists. In particular the sum is at least $i_{j+1} - i_j = \frac{n}{k}$. If no essential level exists then the sum is zero and the bound is still true. 

     Recall that $N_i$ denotes the event that level $i$ is normal. Since $N_i \subset E_i$, we deduce that:
    \begin{equation}
        \label{ineq:local_normal_jumps}
        \sum_{i=i_j}^{i_{j+1}-1} \mathds{1}_{N_i}(1 + F_i + \widetilde{\EF}_i) \le \frac{n}{k}.
    \end{equation}

    Note that $\E[\widetilde{\EF_i} \mid N_i] \ge \Pr(\EF_i \ge 1 \mid N_i) \ge m_n(\frac{i}{n})$.
    Using $(F_i \mid N_i) \preceq \mathcal{G}(\frac{1}{2}) - 1$, this yields for all level $i$:
    $$\E[\mathds{1}_{N_i}(1 + F_i + \widetilde{\EF_i})] \ge \E[\mathds{1}_{N_i}](2 + m_n(\frac{i}{n})). $$
    
    Hence, by definition of $\eps_{M_n}$, for $i_j \le i \le i_{j+1} - 1$,
    $$ 
    \E[\mathds{1}_{N_i}(1 + F_i + \EF_i)] \ge \E[\mathds{1}_{N_i}](2 + m_n(\frac{i_j}{n}) - \eps_{m_n}) = \E[\mathds{1}_{N_i}](2 + m_n(\frac{j}{k}) - \eps_{m_n}).
    $$

    Plugging this last inequality into (\ref{ineq:local_normal_jumps}), we get the desired bound for the expected number of essential levels in the local window $i_j \le i < i_{j+1}$.
    \begin{equation*}
        \sum_{i=i_j}^{i_{j+1} - 1} \E[\mathds{1}_{N_i}] \le \frac{n}{k(2 + m_n(\frac{j}{k}) -\eps_{m_n})}
    \end{equation*}
    By Lemma~\ref{lem:decomposition_normal_rest} and the fact that $T^\Out_i - T^c_i \sim \mathcal{G}((1-\frac{\chi}{n})^i\frac{\chi}{n})$, it suffices to bound $\sum_{i = 0}^{n-1} \E[\mathds{1}_{N_i}\mathcal{G}((1-\frac{\chi}{n})^i\frac{\chi}{n})]$. We decompose this sum in smaller sums over the local windows:
    \begin{align}
        \label{ineq:integral_trick}
        \begin{split}
        \sum_{i = 0}^{n-1} \E[\mathds{1}_{N_i}\mathcal{G}((1-\frac{\chi}{n})^i\frac{\chi}{n})] &= \sum_{j = 0}^{k-1}\sum_{i = i_j}^{i_{j+1} - 1} \E[\mathds{1}_{N_i}\mathcal{G}((1-\frac{\chi}{n})^i\frac{\chi}{n})]\\
        &\le \sum_{j = 0}^{k-1}\sum_{i = i_j}^{i_{j+1} - 1} \E[\mathds{1}_{N_i}](1-\frac{\chi}{n})^{-i}\frac{n}{\chi}\\
        &\le \sum_{j = 0}^{k-1} \frac{n}{k(2 + m_n(\frac{j}{k}) - \eps_{m_n})}\frac{n}{\chi}(1-\frac{\chi}{n})^{-{i_j}} \\
        &\le \frac{n^2}{\chi}\left(\frac{1}{k}\sum_{j = 0}^{k-1}\frac{1}{2 + m_n(\frac{j}{k})+ \eps_{m_n}}((1-\frac{\chi}{n})^{-n})^{\frac{j}{k}})\right).
        \end{split}
    \end{align}

    Thanks to the uniform convergence of $m_n$ to $m$, and of the function $x \rightarrow (1-\frac{\chi}{n})^{-nx}$ to to the function $x \rightarrow e^{\chi x}$ over $x\in[0, 1]$ when $n \longrightarrow \infty$, the term between parentheses converges to $\int_{0}^1\frac{e^{\chi x}}{2 + m(x)}$ as $k \to\infty$. This concludes the proof.\qed
\end{proof}    

The proof of Lemma \ref{lem:freeridersspeeduplb} is very similar, but requires a little bit more work to justify that exceedingly large jumps cannot occur.

\begin{proof}[of Lemma \ref{lem:freeridersspeeduplb}]
     The proof follows the same steps than that of Lemma \ref{lem:freeridersspeedup}, but we need to work conditioned on $A$.
    As in the previous proof, for $0 \le j \le k$, define $i_j = \frac{jk}{n}$.

    Conditioned on $A$, the first fitness level $L_j$ reached in the window $[i_j, i_{j+1}]$ is at most $i_j + o\left(\frac{n}{\max(\mu \log \mu, k)}\right)$. Following an analogous argument than for inequality (\ref{ineq:local_essentials}), we have for all $0 \le j \le k - 1$:
    \begin{equation}
        \label{ineq:local_essentialslb}
        \frac{n}{k} - o\left(\frac{n}{\max(\mu \log \mu, k)}\right) - 1 = \frac{n + o(n)}{k} \le \sum_{i=i_j}^{i_{j+1}-1} \mathds{1}_{E_i}(1 + F_i + \EF_i).
    \end{equation}

    Now, we need to argue that strange fitness levels do not contribute too much to the above quantity in expectation.
    This contribution is exactly:
    $$\sum_{i=i_j}^{i_{j+1}-1} \E[\mathds{1}_{S_i}(1 + F_i + \EF_i) \mid A] = \sum_{i=i_j}^{i_{j+1}-1} \Pr(S_i \mid A)(1 + \E[F_i \mid S_i, A] + \E[\EF_i \mid S_i, A]).$$
    The second factor of each summand is simply $\E[\ESucc_i - i \mid A, S_i]$, which is $ o(\frac{n}{\mu \log \mu})$ by definition of $A$. Moreover, by Lemma \ref{lem:no_strange} $\Pr(S_i \mid A) \le \frac{\Pr(S_i)}{P(A)} = \frac{P(S_i)}{1-o(1)} = o(\frac{\mu \log \mu}{n})$. This yields that the total contribution of strange levels is $o(1)(i_{j+1} - i_j) = o(\frac{n}{k})$. Hence, conditioned on $A$, the lower bound (\ref{ineq:local_essentialslb}) that holds for essential levels, also holds for normal levels.
    \begin{equation}
        \label{ineq:local_normal_jumpslb}
        \frac{n + o(n)}{k} \le \sum_{i=i_j}^{i_{j+1}-1} \mathds{1}_{N_i}(1 + F_i + \EF_i).
    \end{equation}
    By an analogous computation to the proof of Lemma \ref{lem:freeridersspeedup}, we obtain:
    $$\sum_{i = 0}^{n-1} \E[\mathds{1}_{N_i}\mathcal{G}(1 -\frac{\chi}{n})^i \frac{\chi}{n}] \ge \frac{n^2 + o(n^2)}{\chi}\int_{0}^1\frac{e^{\chi x}}{2 + M(x)}dx.$$

    Finally, 
    \begin{align*}
        \E[T^\mu] &\ge P(A)\E[T^\mu \mid A]\\
        &\ge (1 - o(1))\sum_{i = 0}^{n-1} \E[\mathds{1}_{N_i}\mathcal{G}(1 -\frac{\chi}{n})^i \frac{\chi}{n}]\\
        &\ge \frac{n^2 + o(n^2)}{\chi}\int_{0}^1\frac{e^{\chi x}}{2 + M(x)}dx,
    \end{align*}
    following the same steps as in \ref{ineq:integral_trick}.\qed

\end{proof}

\subsection{The vanilla \mupoga is not faster than the vanilla \oea for $\mu \in O(\sqrt{n}/\log^2 n)$.}
\label{subsec:vanillanospeedup}
In this section, we show that one cannot expect speedups in using the vanilla \mupoga compared to the \oea for any population size $\mu = O(\sqrt{n}/\log^2 n)$. We call "vanilla" the implementation of the \mupoga consisting in standard bit-wise mutation, uniform crossover and uniform tie-breaker. This absence of a substantial speedup stems from the fact that such populations reach a natural diversity equilibrium which is not big enough to be exploited by crossover. This leads to a small number of expected extra free-riders, and hence, by Lemma \ref{lem:freeridersspeeduplb}, a negligible speedup. 

We will first show inductively that with high probability there are no gaps larger than $O(\log n)$ between normal fitness levels. To this end, we will make the simplifying assumption that the population is initialized with the all-0 bit-strings. This slight cheat conveniently provides us with the base case of the induction. We do not believe that this is really needed, but did not see an easy way to remove this assumption.


We will first show that the expected diversity when leaving a given fitness level is small.

\begin{lemma}
    \label{lem:lowmulowdiv}
    Consider the \mupoga with population size $\mu = o(\sqrt{n})$. Let $1 \le i \le n - \omega(\mu)$ and suppose that at some timestep $t_0 \ge 0$, $P_{t_0}$ is consolidated, and $f(P_{t_0}) = i$. Then,
    $$\E[d_{T^o_i}] \le O\left(d(P_{t_0})\frac{\mu^2}{n} + \frac{\mu}{n-i}\right).$$
\end{lemma}

\begin{proof}
    We rely on a result from~\cite{lengler2023analysing} (Corollary 3.4 and Theorem~4.3) which we restate here in a simplified form:
    \begin{lemma}
        \label{lem:lenglerdiv}
        Consider a population $Q_t = \{x_1, \dots, x_\mu \}$ of individuals of size $N > 0$. Consider any process that
        \begin{enumerate}
            \item creates $y$ by any random procedure such that $\E(S(y)) = S(Q_t)/\mu;$
            \item creates $y^\prime$ from $y$ by an unbiased mutation operator which flips $\chi$ bits in expectation;
            \item sets $Q_{t+1} = Q_t \cup \{y^\prime\} \backslash \{x_d\}$ for a uniformly random $1 \le d \le \mu$. 
        \end{enumerate}
        Then,
            $$\E [S(Q_{t+1}) \mid S(Q_t)]=\left(1 - \frac{2}{\mu^2} - \frac{4(\mu - 1)\chi}{\mu^2N}\right)S(Q_t) + 2(\mu-1)\chi.$$
        Moreover, let $T_d = \inf \{t \mid S(Q_t) \le 2\alpha\}$, where $\alpha := \frac{2(\mu - 1)\mu^2\chi n}{n + 4(\mu - 1)\chi}$. If $\mu=o(n)$, then 
    $$\E[T_d] = O(\mu^2\log n).$$
    \end{lemma}
     The population $P_t$ of the vanilla \mupoga running on \leadingones clearly verifies hypotheses 1 and 2. However, it does not verify hypothesis 3 as is. To make use of Lemma \ref{lem:lenglerdiv}, we need to do some rescaling. First, we fix a normal level $0 \le i \le n-\omega(\mu)$. We denote by $T^\val_j$ the $j$-th time where a valid offspring is generated after $T^\In_i$. Note that $T^\val_0$ is equal to $T^\In_i$. Then, we define, for $t \ge 0$, $Q_t = \{x_{[F_t + 2:]}, x \in P_{T^\val_t}\}$. The population $Q_t$ now verifies the hypotheses for Lemma \ref{lem:lenglerdiv}, for $t$ such that $T^\val_t \le T^\Out_i$. Let $t^*$ be the last $t$ such that $T^\val_{t^*} \le T_i$. Then, $ \mathcal{G}(\frac{\chi}{n}) \preceq t^*$: since $i$ is normal, we know that it will be left by mutation, and every validated offspring produced via mutation has probability $\frac{\chi}{n}$ to have fitness higher than $i$.  Denote $n_i = n - i$, and note that $n_i = \omega(\mu)$ by hypothesis. For $t \le t^*$:
    \begin{equation*}
        \E[S(Q_{t+1}) \mid S(Q_t), t^*] = \left(1 - \frac{2}{\mu^2} + \frac{4(\mu - 1)\chi}{\mu^2n_i}\right)S(Q_t) + 2(\mu - 1)\chi,
    \end{equation*} 
    where we note that $\frac{4(\mu - 1)\chi}{\mu^2n_i} = o(\frac{1}{\mu^2})$.
    Denote by $\alpha$ the fixed point of this arithmetic-geometric recursion. It is easy to see that $\alpha = \Theta(\mu^3)$, see~\cite{lengler2023analysing}.
    
    Solving the recursion yields:
    \begin{align*}
    &\E[S(Q_{t^*}) \mid t^*] = (S(P_{t_0}) - \alpha)\left(1 - \frac{2}{\mu^2} + o\left( \frac{1}{\mu^2} \right)\right)^{t^*} + \alpha,
    \end{align*}

    Thus, by the law of total expectation: 
    \begin{align*}
        \E[S(Q_{t^*})] &\le \sum_{t = 1}^{\infty}\frac{\chi}{n}(1-\frac{\chi}{n})^{t-1}\alpha\left[ 1 - (1-\frac{2}{\mu^2} + o(\frac{1}{\mu^2}))\right]^{t*} \\
        &+ S(P_{t_0})\sum_{t = 1}^{\infty}\frac{\chi}{n}(1-\frac{\chi}{n})^{t-1}\left[1-\frac{2}{\mu^2} + o(\frac{1}{\mu^2})\right]^{t^*}.
    \end{align*}

    Noticing that, by virtue of $\mu^2 = o(n)$,
    $$
    \sum_{t = 1}^{\infty}\frac{\chi}{n}(1-\frac{\chi}{n})^{t-1}\left[1-\frac{2}{\mu^2} + o(\frac{1}{\mu^2})\right]^{t^*} = \frac{\chi}{n}\frac{1 - \frac{2}{\mu^2} + o(\frac{1}{\mu^2})}{1-(1-\frac{\chi}{n})(1-\frac{2}{\mu^2} + o(\frac{1}{\mu^2}))} = O(\frac{\mu^2}{n}),
    $$
    one obtains:
    \begin{align*}
        &\E[S(Q_{t^*})] \le \alpha(1 - O(\frac{\mu^2}{n})) + O(S(P_{t_0})\frac{\mu^2}{n}) = O(\mu^3 + S(P_{t_0})\frac{\mu^2}{n}).
    \end{align*}
    Since $S(Q_t^*) = S(P_{T^o_i})$, this yields:
    
    \begin{align*}\E[d_{T^o_i}] \le \frac{\E[S(Q_{t^*})]}{O(\mu^2n_i)} = O(\frac{\mu}{n_i} + d(P_{t_0})\frac{\mu^2} {n}).\tag*{\qed}
    \end{align*}
\end{proof}

Next we show inductively that the algorithm frequently encounters consolidated fitness levels with small diversity. We use the following terminology.

\begin{definition}
    We call a population $P$ of size $\mu$ \emph{good} if $P$ is consolidated (all individuals have the same fitness) and $d(P) \le 1/\mu$. For the \mupoga on \leadingones, we call a fitness level $i$ \emph{good} if there exists a good population $P_t$ with $f(P_t) = i$.
\end{definition}

We can now show by induction that good fitness levels are never more than $O(\log n)$ apart. The inductive step is provided by the following lemma.

\begin{lemma}\label{lem:good_levels_inductive_step}
    Suppose that $\mu = O(\sqrt{n}/\log^2 n)$ and that $i = n-\omega(\mu^2)$ is a good fitness level. Then 
    with probability $1 - o(1/n)$ the next good fitness level $i'$ satisfies $i' = i+O(\log n)$.
\end{lemma}
%
\begin{proof}
Consider a good population $P_t$ on level $i$ and a position $j>i+1$. Recall that if not all individuals are identical in position $j$, then it contributes at least $2\mu-2$ to the diversity. Hence, the number of such $j$ in which not all individuals are identical is at most $S(P_t)/(2\mu-2) = \tfrac\mu2 (n-i-1) d_t \le (n-i-1)/2$ since $P_t$ was good. Hence, at least half of the $n-i-1$ non-optimized positions are identical in all individuals in $P_t$. By symmetry, each such position has probability exactly $1/2$ that the whole population has a zero-bit. Let us call such positions \emph{stopping positions}. By the Chernoff bound, for a suitable constant $C>0$ the number of stopping positions among the next $C\log n$ fitness levels $i+2,\ldots,i+C\log n +1$ is at least $\log n$ with probability $1-o(1/n)$.

Now consider the next $\tau := C^2\mu^2\log^3 n = o(n)$ generations. The probability that a fixed stopping position is flipped by at least one mutation during that time is $o(1)$. Hence, the probability that all $\log n$ stopping positions are flipped at least once is $o(1/n)$. Hence, with probability $1-o(1/n)$ the fitness will not exceed $i+C\log n$ during those $\tau$ generations. By pigeonhole principle there must be a fitness level $j$ on which the algorithm spends at least $\tau' := \tau/(C\log n) = C \mu^2 \log^2 n$ generations. In the following we will analyze the behaviour on this level $j$.

Once the fitness level $j$ is reached by at least one individual, with probability $1-o(1/n)$ it takes $O(\mu\log \mu) = o(\tau')$ generations to consolidate the population on fitness level $j$, by Lemma~\ref{lem:stayingisconsolidating}. In Lemma~\ref{lem:lenglerdiv} we considered the time $T_d$ until the diversity on such a level drops below $2\alpha$ for $\alpha := \frac{2(\mu - 1)\mu^2\chi n}{n + 4(\mu - 1)\chi} = O(\mu^3)$, and stated that $\E[T_d] =O(\mu^2 \log n)$. Now we split the time after consolidation on level $j$ into intervals of lengths $2\E[T_d]$. In each such interval, by Markov's inequality the diversity drops below $2\alpha$ with probability at least $1/2$. The number of such intervals is at least $\Omega(\tau' /\E[T_d]) = \Omega(\log n)$, where we can make the hidden constant as big as we like due to the factor $C$ in $\tau'$. Hence, the probability that the diversity drops below $2\alpha$ in at least one such phase is $1-o(1/n)$. The lemma now follows because when the diversity of $P_t$ drops to $S(P_t) \le 2\alpha$ then $d(P_t) = \frac{S(P_t)}{\mu(\mu-1)(n-i-1)} = O(\frac{\mu}{n-i+1}) = o(1/\mu)$, which implies that $P_t$ is a good population. Hence, level $j$ is good with probability at least $1-o(1/n)$.\qed
\end{proof}

Applying Lemma~\ref{lem:good_levels_inductive_step} inductively immediately yields the following corollary, which automatically satisfies the last hypothesis of Lemma \ref{lem:freeridersspeeduplb}: 
\begin{corollary}
    Consider the \mupoga with $\mu = O(\sqrt{n}/\log^2 n)$ starting in the all-zero string. Then, with probability $1-o(1)$, $\forall 1 \le i \le n, \ESucc_i - i = O(\log n)$.
\end{corollary}

In what follows, we work conditionally on the event $J: "\forall 1 \le i \le n, \ESucc_i - i = O(\log n)"$. Note that since $J$ has high probability by the previous corollary, we only need to prove the hypotheses of Lemma \ref{lem:freeridersspeeduplb} conditional on $J$, by virtue of $\E[T^\mu] \ge \E[T^\mu \mid J]\Pr(J)$.

Now, we can prove that the expected number of extra free-riders is negligible.

\begin{lemma}
        Consider the \mupoga with $\mu = O(\sqrt{n}/\log^2 n)$ starting in the all-zero string. Then for all $1 \le i \le n-\omega(\mu^2)$, $\E[\EF_i \mid N_i] = o(1)$. 
\end{lemma}

\begin{proof}
    Consider any normal fitness level $1\le i \le n$. The structure of the proof resembles the structure of the proof of Lemma~\ref{lem:good_levels_inductive_step}, but this time we do not need so small error probabilities because even in the worst case we only need to consider the $O(\log n)$ levels between normal fitness levels. 
    We define the population $Q_t$ for $t \ge 0$, and $t^*$ as in the proof of Lemma \ref{lem:lowmulowdiv}. 
    We know that $\mathcal{G}({\frac{\chi}{n}}) \preceq t^*$, and hence, if we let $T = o(\frac{n}{\log n})$, we have $\Pr(t^* \le T) = o(\frac{1}{\log n})$. Let us call $E$ the event that $t^* > T$.
    Note that $\mu^2 \log n \log \log n = o(T)$, hence we may split $T$ into $\omega(\log \log n)$ phases of size $\mu^2 \log n$ intervals each. Conditioned on $E$, similarly to the proof of Lemma~\ref{lem:good_levels_inductive_step} and Markov's inequality, each interval has a constant probability of reaching $d(P_t) \le \frac{2\alpha}{\mu^2n_i} = O(\frac{\mu}{n_i})$, where $\alpha$ is given in Lemma~\ref{lem:lenglerdiv}. Hence, with probability $1 - o(\frac{1}{\log n})$, at some timestep $t_0$, $d(P_{t_0}) = O(\frac{\mu}{n_i})$. Let us call this event $F$.
    By Lemma \ref{lem:lowmulowdiv}, we obtain: 
    $$\E[d_{T^o_i} \mid E, F] = O(\frac{\mu}{n_i}).$$

    By the law of total expectation: 

    $$\E[\EF_i] = \E[\EF_i \mid \bar{E}]\Pr(\bar E) +  \E[\EF_i \mid E, \bar{F}]\Pr(E, \bar{F}) + \E[\EF_i \mid E, F]\Pr(E, F)$$

    For the first term, note that since we conditioned on $J$, $\E[\EF_i \mid \bar{E}] = O(\log n)$, and that $P(\bar E) = o(\frac{1}{\log n})$. Hence, this term is $o(1)$.
    
    For the second term, again $ \E[\EF_i \mid E, \bar{F}] = O(\log n)$, and $\Pr(E, \bar{F}) \le \Pr(\bar{F} \mid E) = o(\frac{1}{\log n})$, so this term is again $o(1)$.
    
    As for the third term, notice that if the column $\Succ_i + 1$ is a stopping column, that is, a column where all individuals in the population have a $0$-bit, then with high probability $\Succ_i + 1$ is normal, and in particular $\EF_i = 0$. 
    
    Indeed, in that case, the probability of leaving fitness level $i$ before consolidation is at most the probability that a mutation in column $\Succ_i + 1$ happens before consolidation, which we already derived to be $O(\frac{\mu \log \mu}{n})$ in the proof of Lemma \ref{lem:no_strange}.

    It remains to prove that, conditioned on $E \cap F$, $\Succ_i + 1$ is a stopping column with high probability. Denote by $x$ the individual that reaches fitness $\Succ_i$ first. By definition, $x_{\Succ_i + 1} = 0$. Hence, the probability that $\Succ_i + 1$ is not a stopping column is the probability that at least one individual in the population is different that $x$ on column $\Succ_i$ at time $T^{\In}_{\Succ_i}$. At time $T^o_i$, the probability that a given individual is different than $x$ on this column is $d_{T^o_i}$ by Corollary \ref{cor:unbiased}. Hence, at time $T^{\In}_{\Succ_i}$, it is at most $d_{T^o_i} + O(\frac{1}{n}) = O(\frac{\mu}{n_i})$, where the additive $O(1/n)$ takes into account that $x$ might be generated by mutating the $\Succ_i + 1$-th bit. Hence, the expected number $N$ of individuals different than $x$ on this column is
    $$\E[N] = (\mu - 1)O(\frac{\mu}{n_i}) = O(\frac{\mu^2}{n_i}) = o(1).$$
    We conclude using the fact that $\Pr(N \ge 1) \le \E[N]$.\qed
    
\end{proof}

\subsection{Using a diversity improving tie-breaker speeds up the \tpoga on \leadingones by a constant factor.}
In this section, we introduce a particular tie-breaker, and show that it ensures a constant factor improvement for which we derive a lower bound. Finally, we show that using an adaptative crossover probability leads to further improvement.

Throughout the whole section, we will denote $P = \{x_1, x_2\}$ with $x_2 \preceq x_1$. If there is an ambiguity on the value of $t$, we will write $P_t = \{x_1^t, x_2^t\}$. 
    \label{subsec:tiebreak}
    In this section, we introduce a tie-breaking mechanism that enhances the efficiency of crossovers by keeping the two individuals in the population genetically diverse. When two individuals are in a tie, this tie-breaker chooses the one that has the highest $S$-value among the rest of the population (that is, the one furthest away from the rest of the population in Hamming distance). This has the effect of keeping a population with a high $d_t$ value, and hopefully that crossover is more likely to bring interesting novelties in the genotype of the offspring. Here is a pseudocode for this tie-breaker:
    \begin{algorithm2e}
        \caption{The diversity improving tie-breaker.}
        \Input{Q: a set of bit-strings of size $n$. \\
        x, y: two bit strings of length $n$.}
        
        \If{$S_{Q\backslash\{x, y\}}(x) \ge S_{Q\backslash\{x, y\}}(y)$}{
        \Return{x}
        }
        \Else{
        \Return{y}
        }
    \end{algorithm2e}
    
    This tie-breaker is similar to that studied by \cite{DangFKKLOSS18} (Section 5.5), who proved that it significantly enhances the optimization of the \jump functions.
    \begin{observation}
        This tie-breaker is unbiased because automorphisms of the hypercube preserve the Hamming distance.
    \end{observation}
    Now, we observe that this tie-breaker indeeds improves $d_t$ when the offspring produced at $t$ is not better than one of the individuals.
    \begin{lemma}
        Consider a run of the \mupoga for any population size $\mu$, using the diversity improving tie-breaker. Suppose that the offspring $y$ produced at time $t$ is not fitter than any of the individuals in $P_t$. Then, $d_{t+1} \ge d_t$.
    \end{lemma}
    
    \begin{proof}
        Since the new offspring $y$ does not reduce the size of the non-optimized part of a fit bit-string in the population, we just have to show that $S(P_{t+1}) \ge S(P_t)$. 
        Suppose that the newly generated offspring is introduced in the population $P_{t+1}$ (otherwise the result is trivial). This means that $P_{t+1} = P_t \backslash \{x\} \cup \{y\}$ for some $x \in P_t$ such that 
        $$
        S_{P_t \backslash \{x\}}(x) \le S_{P_t \backslash \{x\}}(y)
        $$
        On the other hand,
        \begin{align*}
            S(P_{t+1}) &= \sum_{z \in P_t, z \neq x}H(y, z) + \sum_{z \in P_t, z \neq x} S_{P_t \backslash \{x\}}(z) + H(z, y) \\
            &= 2S_{P_t \backslash \{x\}}(y) + \sum_{z \in P_t, z \neq x}S_{P_t}(z) - H(z, x) \\
            &= 2S_{P_t \backslash \{x\}}(y) - S_{P_t\backslash\{x\}}(x) + \sum_{z \in P_t, z \neq x}S_{P_t}(z) \\
            &= 2(S_{P_t \backslash \{x\}}(y) - S_{P_t\backslash\{x\}}(x)) + S(P_t).
        \end{align*}
        Applying the previous inequality gives the desired result.\qed
    \end{proof}

    We will now prove the following theorem:
    \begin{theorem}
        \label{th:tpoga_tiebreaker}
        Consider the \tpoga using standard bit-wise mutation, uniform crossover with constant probability $p_c$, and the diversity-improving tie-breaker. Then its runtime $T$ satisfies:
        $$
        \E[T] \le \frac{2}{2 + \frac{p_c(1-p_c)}{12-8p_c}}\cdot \frac{e^{\chi}-1}{2\chi^2}n^2 + o(n^2)
        $$
    \end{theorem}
    Note that the fraction is a constant strictly smaller than $1$, and the second factor is up to a $(1+o(1))$ the runtime of a vanilla \tpoga (with uniform tie-breaking), or a \oea. This means that the diversity-improving tie-breaker brings a constant factor improvement compared to the vanilla \tpoga and the \oea. Note that there is room for improvement in the precise constant, since empirically the speedup seems to be about twice as large.
    
    The first step to prove this theorem is to study the diversity evolution process on a fitness plateau.
    \begin{lemma}
        \label{lem:d_lb_tiebreaker}
        Consider a run of the \tpoga using the diversity-improving tie-breaker and standard bit-wise mutation with any mutation rate $\chi$ and uniform crossover. Let $0 \le i \le n-1$. Then, $$\E[d_{T^\Out_i} \mid N_i] \ge \frac{1-p_c}{3-2p_c} + o(1)$$.
    \end{lemma}
    Let us first make the following observation.
    \begin{observation}
        Assume that the whole population has fitness at least $F_t$. Then every crossover of two parents has also fitness at least $F_t$. Hence, regardless of whether crossover happened or not, the probability that the offspring has fitness at least $F_t$ is $(1-\chi/n)^{F_t}$. By Bayes' theorem, this conversely implies that the event that an offspring has fitness at least $F_t$ does not change the probability that crossover has happened. In formula, for all $t \ge 0$, if we denote by $y_t$ the offspring produced at time $t$:  
        $$\Pr(C_t \mid f(y_t) \ge F_t, t \ge T^c_{F_t}) = p_c.$$ 
    \end{observation}

    We may now prove the lemma.
    \begin{proof}[of Lemma~\ref{lem:d_lb_tiebreaker}]
        Let $0 \le i \le n-1$ be a normal fitness level. Define $T^\fit_t$ to be the time when the $t$-th offspring with fitness equal to $i$ or higher is generated after consolidation.
    Define the bit-string $b_t = \neg ((x^{T^\fit_t}_{1})_{[i+2:]} \oplus (x^{T^\fit_t}_{2})_{[i+2:]})$, so $b_t$ has a one-bit if $x^{T^\fit_t}_1$ and $x^{T^\fit_t}_2$ coincide in their non-optimized part, and a zero-bit otherwise. Define $t^*$ to be the last $t$ such that $T^\fit_t \le T^\Out_i$. Note that $\OM_t = \onemax(b_t) = (n-i-1)d_{T^\fit_t}$, so $\OM_{t^*} = n_id_{T^\Out_i}$, where $n_i = n - i - 1$. 

Now, note that when the offspring generated at $T^\fit_t$ is generated without crossover (probability  $1 - p_c$ by the last observation), the increase of $\OM_t$ in one time step is always bigger than the difference between the number of mutations that happened in the set $0_{b_t}$ and in the set $1_{b_t}$, which have respectively an expectation of $\frac{\chi}{n}(n_i - \OM_t)$, and $\frac{\chi}{n}\OM_t$. In every other case, the tie-breaker ensures that $\OM_t$ increases.
    This gives:
    \begin{align*}
        &\E[\OM_{t+1} - \OM_t \mid \OM_t] \ge (1-p_c)(n_i - 2\OM_t)\frac{\chi}{n},
    \end{align*}
    or equivalently
     \begin{align*}
        &\E[\OM_{t+1} \mid \OM_t] \ge \OM_t(1 - 2\frac{\chi}{n}(1-p_c)) + n_i(1-p_c)\frac{\chi}{n}.
    \end{align*}
    By the law of total expectation: 
    $$
    \E[\OM_{t+1} \mid \OM_0] \ge \E[\OM_t \mid \OM_0](1 - 2\frac{\chi}{n}(1-p_c)) +  n_i(1-p_c)\frac{\chi}{n}.
    $$
    Solving the arithmetic-geometric recursion yields:
    $$
    \E[\OM_t \mid \OM_0] \ge (\OM_0 - \frac{n_i}{2})(1 -  2\frac{\chi}{n}(1-p_c))^t + \frac{n_i}{2}.
    $$
    Since $\OM_0 \ge 0$, we get:
    $$
    \E[\OM_{t^*} \mid t^*] \ge \frac{n_i}{2}\left(1 - \left(1 - \frac{2\chi}{n}(1-p_c)\right)^{t^*}\right).
    $$
    Moreover, $t^* \sim \mathcal{G}(\frac{\chi}{n})$. Indeed, every generated offspring has probability $\frac{\chi}{n}$ to leave level $i$, regardless of if crossover was used or not to generate it, since we always apply mutation after crossover.
    Hence, by the law of total expectation: 
    \begin{align*}
        \E[\OM_{t^*}] &\ge \sum_{t = 1}^{\infty}\frac{\chi}{n}\left(1-\frac{\chi}{n}\right)^{t-1}\frac{n_i}{2}\left(1 - \left(1 - \frac{2\chi}{n}(1-p_c)\right)^t\right) \\
        &\ge \frac{n_i}{2}\left[1 - \left(1-\frac{2\chi}{n}(1-p_c)\right)\frac{\chi}{n}\sum_{t = 0}^{\infty}\left(\left(1-\frac{\chi}{n}\right)\left(1 - \frac{2\chi}{n}(1-p_c)\right) \right)^t\right]
         \\
        &\ge \frac{n_i}{2}\left[ 1 - \left(1-\frac{2\chi}{n}(1-p_c)\right)\frac{\frac{\chi}{n}}{1 - (1 - \frac{\chi}{n})(1 - \frac{2\chi}{n}(1-p_c)))}\right] \\
        &\ge \frac{n_i}{2}\left[1 - \frac{1}{3 - 2p_c} + o(1)\right] \\
        &\ge n_i\frac{1-p_c}{3-2p_c} + o(n_i).
    \end{align*}
    Dividing by $n_i$ yields:
    \begin{align*}
    \E[d_{T^\Out_i}] \ge \frac{1-p_c}{3-2p_c} + o(1).\tag*{\qed}
    \end{align*}
    \end{proof}

    Now, we argue that having this constant lower bound for the diversity when leaving $i$ guarantees a constant lower bound for the expectation of $\EF_i$. 

    \begin{lemma}
        \label{lem:EF_lb_tiebreaker}
        For $0 \le i \le n$,
        $$
        \Pr[\EF_i \ge 1 \mid N_i, d_{T^\Out_i}] \ge \frac{p_c}{8}d_{T^\Out_i},
        $$
    \end{lemma}
    \begin{proof}
        Let $i$ be a normal fitness level. Suppose $F_i = 0$, that is $\Succ_i = i+1$, which happens with probability $\frac{1}{2}$ by Corollary \ref{cor:small_fitness_improvements}. The probability that $x_1$ and $x_2$ are different at $i+2$ is $d_{T^\Out_i}$ at time $T^\Out_i$. At $T^\In_{i+1}$, it is still $d_{T^\Out_i}$. Suppose it is the case. We want to lower bound the probability that an offspring of fitness $i+1$ is generated by crossover before a copy of $x_1$ is generated, from the moment when $x_1$ reaches $i+1$. Let us call this event $A$. When an offspring with fitness larger than $i$ is generated, the probability that it was generated by crossover is $p_c$. Conditioned on this, the probability that the offspring has fitness higher than $i+1$ is the probability that bit $i+1$ is set correctly, which is $\frac{1}{2}$.
        Hence, $\Pr(A) \ge p_C/2$. 
        Conditioned on $A$, the probability that the crossover that realized $A$ also gets the extra free-rider is $\frac{1}{2}$.
        Hence,
        \begin{align*}
        \Pr[\EF_i \ge 1 \mid N_i, d_{T^\Out_i}] & \ge \Pr(x^{T^\In_{\Succ_i}}_{1_{\Succ_i + 1}} \neq x^{T^\In_{\Succ_i}}_{2_{\Succ_i + 1}})\Pr(\Succ_i = i+1)\Pr(A)\frac{1}{2} \\
        &\ge \frac{p_c}{8}d_{T^\Out_i}.\tag*{\qed}
        \end{align*}
    \end{proof}


    We conclude the proof of Theorem \ref{th:tpoga_tiebreaker} by using the law of total expectation which gives $\E[\EF_i \mid N_i] \ge \frac{p_c(1-p_c)}{12 - 8p_c}$, together with Lemma \ref{lem:freeridersspeedup}, which can be applied for $m_n = \frac{p_c(1-p_c)}{12 - 8p_c}$. \qed

    This bound is optimized for $p_c \approx 0.6$, which is experimentally the optimal static crossover value. However, our proof naturally suggests an adaptive mechanism for the crossover probability, which is to set $p_c = 1$ when the population is not consolidated, and $p_c = 0$ otherwise. Indeed, in Lemma \ref{lem:d_lb_tiebreaker}, we see that the lower bound for the diversity decreases with $p_c$. This is intuitive, as crossover hinders the accumulation of diversity when in the consolidated phase: if a crossover is made between two individuals with Hamming distance $d$, the resulting offspring is at distance lower than $d$ from both its parents with overwhelming probability, and thus will not be accepted. Conversely, the lower bound for the probability to get an extra free-rider conditional on $d_{T^o_i}$ in the non-consolidated phase is increasing in $p_c$. Even if this bound is not tight, this is again intuitive, as the more crossover one does, the more trials one gets to obtain an extra free-rider.

\section{Conclusion}\label{sec:conclusion}
In this work, we examined the connection between population diversity and progress on the \leadingones problem by the \mupoga. We have shown that the naturally evolving diversity for any $\mu = o(\sqrt{n}/\log n)$ is not enough to improve the runtime by more than a $(1+o(1))$ factor. On the other hand, even for $\mu=2$ simple tie-breaking in favor of diversity leads to so much diversity in the population that the runtime decreases by a constant factor. 

%

There are many question that we had to leave open. The most interesting is what happens for $\mu = \Omega(\sqrt{n}) \cap o(n/\log n)$. We conjecture that the vanilla version can not create enough diversity to give a constant factor speed-up, even though the average Hamming distance between parents increases further. We conjecture further that the problem is not that differences in bits are never created, but instead the problem is that they also get lost again. So, if there was a way of preventing these differences to get lost, we may even hope for algorithms which get extra free riders on \emph{most} fitness levels, which could lead to asymptotically optimization time $o(n^2)$. We believe that this is a very interesting setting to explore more systematically potential diversity-preserving mechanisms.



\medskip
\noindent \textbf{Disclosure of Interests}. The author do not have conflicting interests. 

{
%
\bibliographystyle{splncs04}
\bibliography{references}
}
\end{document}